%% file: paper.tex
\documentclass[12pt]{article} % Anonymized submission
\usepackage{fullpage}
\usepackage{algorithm,algorithmic}
\usepackage{amsmath}
\usepackage{amsthm}
\usepackage{amsfonts}
\usepackage{latexsym}
\usepackage{amssymb}
\usepackage{hyperref}
\usepackage[round]{natbib}

\input{macro}

\newtheorem{theorem}{Theorem}
\newtheorem{lemma}{Lemma}

\hypersetup{colorlinks,
            linkcolor=blue,
            citecolor=blue,
            urlcolor=magenta,
            linktocpage,
            plainpages=false}

\title{Online Learning with Composite Loss Functions}

\author{%
\makebox[0.4\linewidth]{Ofer Dekel}\\
Microsoft Research\\
\texttt{oferd@microsoft.com}
\and
\makebox[0.4\linewidth]{Jian Ding\thanks{Most of this work was done while the author was at Microsoft Research, Redmond.}}\\
University of Chicago\\
\texttt{jianding@galton.uchicago.edu}
\and\\
\makebox[0.4\linewidth]{Tomer Koren\footnotemark[1]}\\
Technion\\
\texttt{tomerk@technion.ac.il}
\and\\
\makebox[0.4\linewidth]{Yuval Peres}\\
Microsoft Research\\
\texttt{peres@microsoft.com}
}

\date{}

\begin{document}

\maketitle

\begin{abstract}%   <- trailing '%' for backward compatibility of .sty file
We study a new class of online learning problems where each of the
online algorithm's actions is assigned an adversarial value, and the
loss of the algorithm at each step is a known and deterministic function 
of the values assigned to its recent actions. 
This class includes problems where the algorithm's loss is
the \emph{minimum} over the recent adversarial values, the
\emph{maximum} over the recent values, or a \emph{linear combination}
of the recent values. We analyze the minimax regret of this class of
problems when the algorithm receives bandit feedback, and prove that
when the    \emph{minimum} or \emph{maximum} functions are used, the
minimax regret is $\widetilde \Omega(T^{2/3})$ (so called \emph{hard}
online learning problems), and when a linear function is used, the
minimax regret is $\widetilde O(\sqrt{T})$ (so called \emph{easy}
learning problems). Previously, the only online learning problem that
was known to be provably hard was the multi-armed bandit with
switching costs.
%The present work provides the second natural online learning setting that contains provably hard problems
\end{abstract}

%%%% COLT format
%\begin{keywords}
%Online Learning, Adaptive Adversaries, Switching Costs, Lower Bounds
%\end{keywords}

\section{Introduction}

Online learning is often described as a $T$-round repeated game
between a randomized player and an adversary. On each round of the
game, the player and the adversary play simultaneously: the player
(randomly) chooses an action from an action set $\Xcal$ while the
adversary assigns a loss value to each action in $\Xcal$. The player
then incurs the loss assigned to the action he chose. At the end of
each round, the adversary sees the player's action and possibly adapts
his strategy. This type of adversary is called an \emph{adaptive}
adversary (sometimes also called \emph{reactive} or
\emph{non-oblivious}). In this paper, we focus on the simplest online
learning setting, where $\Xcal$ is assumed to be the finite set
$\{1,\ldots,k\}$.

The adversary has unlimited computational power and therefore, without
loss of generality, he can prepare his entire strategy in advance by
enumerating over all possible action sequences and predetermining his
response to each one.  More formally, we assume that the adversary
starts the game by choosing a sequence of $T$ history-dependent loss
functions, $f_1,\ldots,f_T$, where each $f_t: \Xcal^t \mapsto [0,1]$
(note that $f_t$ depends on the player's entire history of $t$
actions). With this, the adversary concludes his role in the game and
only the player actively participates in the $T$ rounds. On round $t$,
the player (randomly) chooses an action $X_t$ from the action set
$\Xcal$ and incurs the loss $f_t(X_{1:t})$ (where $X_{1:t}$ is our
shorthand for the sequence $(X_1,\ldots,X_t)$). The player's goal is
to accumulate a small total loss, $\sum_{t=1}^T f_t(X_{1:t})$.

At the end of each round, the player receives some feedback, which he
uses to inform his choices on future rounds.  We distinguish between
different feedback models. The least informative feedback model we
consider is \emph{bandit feedback}, where the player observes his loss
on each round, $f_t(X_{1:t})$, but nothing else. In other words, after
choosing his action, the player receives a single real number. The
prediction game with bandit feedback is commonly known as the
\emph{adversarial multi-armed bandit} problem \citep{Auer:02} and the
actions in $\Xcal$ are called \emph{arms}. A more informative feedback
model is \emph{full feedback} (also called \emph{full information}
feedback), where the player also observes the loss he would have
incurred had he played a different action on the current round. In
other words, the player receives $f_t(X_{1:(t-1)},x)$ for each $x \in
\Xcal$, for a total of $|\Xcal|$ real numbers on each round. The
prediction game with full information is often called \emph{prediction
  with expert advice} \citep{CesaFrHaHeScWa97} and each action is called an
\emph{expert}.
%\tk{i guess we only need to discuss bandit feedback. what happened to the full-feedback reduction?}

A third feedback model, the most informative of the three, is
\emph{counterfactual feedback}. In this model, at the end of round
$t$, the player receives the complete definition of the loss function
$f_t$. In other words, he receives the value of $f_t(x_1,\ldots,x_t)$
for all $(x_1\ldots,x_t) \in \Xcal^t$ (for a total of $|\Xcal|^t$ real
numbers). This form of feedback allows the player to answer questions
of the form ``how would the adversary have acted today had I played
differently in the past?'' This form of feedback is neglected in the
literature, primarily because most of the existing literature focuses
on \emph{oblivious} adversaries (who do not adapt according to the
player's past actions), for which counterfactual feedback is
equivalent to full feedback.

Since the loss functions are adversarial, their values are only
meaningful when compared to an adequate baseline. Therefore, we
evaluate the player using the notion of \emph{policy regret}
\citep{Arora:12}, abbreviated simply as \emph{regret}, and defined as
\begin{equation}
R ~=~ \sum_{t=1}^T f_t(X_1,\ldots,X_t) ~-~ \min_{x \in \Xcal} \sum_{t=1}^T f_t(x,\ldots,x) ~~.
\label{eq:regret}
\end{equation}
Policy regret compares the player's cumulative loss to the loss of the
best policy in hindsight that repeats a single action on all $T$
rounds. The player's goal is to minimize his regret against a
worst-case sequence of loss functions. We note that a different
definition of regret, which we call \emph{standard regret}, is popular
in the literature. However, \citet{Arora:12} showed that standard
regret is completely inadequate for analyzing the performance of
online learning algorithms against adaptive adversaries, so we stick
the definition of regret in \eqref{eq:regret}

While regret measures a specific player's performance against a
specific sequence of loss functions, the inherent difficulty of the
game itself is measured by \emph{minimax regret}. Intuitively, minimax
regret is the expected regret of an optimal player, when he faces an
optimal adversary. More formally, minimax regret is the minimum over
all randomized player strategies, of the maximum over all loss
sequences, of $\E[R]$. If the minimax regret grows sublinearly with
$T$, it implies that the per-round regret rate, $R(T)/T$, must
diminish with the length of the game $T$. In this case, we say that
the game is \emph{learnable}. \citet{Arora:12} showed that without
additional constraints, online learning against an adaptive adversary
has a minimax regret of $\Theta(T)$, and is therefore unlearnable.
This motivates us to weaken the adaptive adversary and study the
minimax regret when we restrict the sequence of loss functions in
different ways.

\paragraph{Easy online learning problems.}

For many years, the standard practice in online learning research was
to find online learning settings for which the minimax regret is
$\t{\Theta}(\sqrt{T})$. 
Following \cite{Antos:12}, we call problems for which the minimax regret is $\t{\Theta}(\sqrt{T})$ \emph{easy} problems.
Initially, minimax regret bounds focused on loss
functions that are generated by an \emph{oblivious} adversary. An
oblivious adversary does not adapt his loss values to the player's
past actions. More formally, this type of adversary first defines a
sequence of single-input functions, $\ell_1,\ldots,\ell_T$, where
each $\ell_t : \Xcal \mapsto [0,1]$, and then sets
$$
\forall ~t \quad f_t(x_1,\ldots,x_t) ~=~ \ell_t(x_t)~~.
$$
When the adversary is oblivious, the definition of regret used in this
paper (\eqref{eq:regret}) and the aforementioned \emph{standard
  regret} are equivalent, so all previous work on oblivious
adversaries is relevant in our setting.  In the full feedback model,
the \emph{Hedge} algorithm \citep{LW94,FS97} and the \emph{Follow the
  Perturbed Leader} algorithm \citep{Kalai:05} both guarantee a
regret of $\t{O}(\sqrt{T})$ on any oblivious loss sequence (where
$\t{O}$ ignores logarithmic terms). A matching
lower bound of $\Omega(\sqrt{T})$ appears in \citet{CesaBianchi:06},
and allows us to conclude that the minimax regret in this setting is
$\t{\Theta}(\sqrt{T})$. In the bandit feedback model, the \emph{Exp3}
algorithm \citep{Auer:02} guarantees a regret of $\t{O}(\sqrt{T})$
against any oblivious loss sequence and implies that the minimax
regret in this setting is also $\t{\Theta}(\sqrt{T})$.

An adversary that is slightly more powerful than an oblivious
adversary is the \emph{switching cost} adversary, who penalizes the
player each time his action is different than the action he chose on
the previous round. Formally, the switching cost adversary starts by
defining a sequence of single-input functions $\ell_1,\ldots,\ell_T$,
where $\ell_t:\Xcal \mapsto [0,1]$, and uses them to set
\begin{equation} \label{eqn:addswitching}
\forall\,t \quad f_t\big(x,x') ~=~ \ell_t(x') + 1\!\!1_{x' \neq x} ~~.
\end{equation}
Note that the range of $f_t$ is $[0,2]$ instead of $[0,1]$; if this is a
problem, it can be easily resolved by replacing $f_t \leftarrow f_t/2$
throughout the analysis. In the full feedback model, the
\emph{Follow the Lazy Leader} algorithm \citep{Kalai:05} and the more
recent \emph{Shrinking Dartboard} algorithm \citep{Geulen:10} both
guarantee a regret of $\t{O}(\sqrt{T})$ against any oblivious sequence
with a switching cost. The $\Omega(\sqrt{T})$ lower bound against
oblivious adversaries holds in this case, and the minimax regret is
therefore $\t{\Theta}(\sqrt{T})$.

The switching cost adversary is a special case of a \emph{$1$-memory}
adversary, who is constrained to choose loss functions that depend
only on the player's last two actions (his current action and the
previous action). More generally, the \emph{$m$-memory} adversary
chooses loss functions that depend on the player's last $m+1$ actions
(the current action plus $m$ previous actions), where $m$ is a
parameter. In the counterfactual feedback model, the work of
\citet{gyorgy2011near} implies that the minimax regret against an
\emph{$m$-memory} adversary is $\t{\Theta}(\sqrt{T})$. 

\paragraph{Hard online learning problems.}

Recently, \citet{CesaBianchiDeSh13,DekelDiKoPe14} showed that online
learning against a switching cost adversary with bandit feedback (more
popularly known as the \emph{multi-armed bandit with switching costs})
has a minimax regret of $\t{\Theta}(T^{2/3})$. This result
proves that there exists a natural\footnote{By \emph{natural}, we mean
  that the problem setting can be described succinctly, and that the
  parameters that define the problem are all independent of $T$. An
  example of an unnatural problem with a minimax regret of
  $\Theta(T^{2/3})$ is the multi-armed bandit problem with $k=T^{1/3}$
  arms.} online learning problem that is learnable, but at a rate that
is substantially slower then $\t{\Theta}(\sqrt{T})$. Again
following \citet{Antos:12}, we say that an online problem is
\emph{hard} if its minimax regret is $\t{\Theta}(T^{2/3})$.

Is the multi-armed bandit with switching costs a one-off example, or
are there other natural hard online learning problems? In this paper,
we answer this question by presenting another hard online learning
setting, which is entirely different than the multi-armed bandit with
switching costs. 

\paragraph{Composite loss functions.}

We define a family of adversaries that generate \emph{composite loss
  functions}.  An adversary in this class is defined by a memory size
$m \ge 0$ and a \emph{loss combining function} $g: [0,1]^{m+1} \mapsto
[0,1]$, both of which are fixed and known to the player. The adversary
starts by defining a sequence of oblivious functions
$\ell_1,\ldots,\ell_T$, where each $\ell_t:\Xcal \mapsto [0,1]$. Then,
he uses $g$ and $\ell_{1:T}$ to define the composite loss functions
$$
\forall~t\quad f_t(x_{1:t}) ~=~ g \big(\ell_{t-m}(x_{t-m}),\ldots,\ell_t(x_t)\big)~~.
$$ 
For completeness, we assume that $\ell_{t} \equiv 0$ for $t \leq 0$.  The
adversary defined above is a special case of a $m$-memory adversary.

For example, we could set $m=1$ and choose the \emph{max} function as
our loss combining function. This choice define a $1$-memory
adversary, with loss functions given by
$$
\forall~t\quad f_t(x_{1:t}) ~=~ \max\big( \ell_{t-1}(x_{t-1}), \ell_t(x_t) \big) ~~.
$$ 
In words, the player's action on each round is given an oblivious
value and the loss at time $t$ is the maximum of the current oblivious
value and previous one.  For brevity, we call this adversary the
\emph{max-adversary}.  The max-adversary can be used to represent
online decision-making scenarios where the player's actions have a
prolonged effect, and a poor choice on round $t$ incurs a penalty on
round $t$ and again on round $t+1$. Similarly, setting $m=1$ and
choosing \emph{min} as the combining function gives the \emph{min
  adversary}. This type of adversary models scenarios where the
environment forgives poor action choices whenever the previous choice
was good. 
Finally, one can also consider choosing a linear function $g$.
Examples of linear combining functions are 
$$
f_t(x_{1:t}) ~=~ \half\big( \ell_{t-1}(x_{t-1}) + \ell_t(x_t) \big) \quad\text{and}\quad
f_t(x_{1:t}) ~=~ \ell_{t-1}(x_{t-1})~~.
$$ 
%
%One of the main technical achievements of this paper is a
%proof that both of these adversaries induce \emph{hard} online
%learning problems when the player receive bandit feedback.

%Another option is to choose a linear function $g$. Examples of 
%linear combining functions are 
%$$
%f_t(x_{1:t}) ~=~ \half\big( \ell_{t-1}(x_{t-1}) + \ell_t(x_t) \big) \quad\text{and}\quad
%f_t(x_{1:t}) ~=~ \ell_{t-1}(x_{t-1})~~.
%$$ 
%

The main technical contribution of this paper is a
$\t{\Omega}(T^{2/3})$ lower bound on the minimax regret against the
\emph{max} and \emph{min} adversaries, showing that each of them
induces a \emph{hard} online learning problem when the player receives
bandit feedback.  In contrast, we show that any linear combining
function induces an \emph{easy} bandit learning problem, with a
minimax regret of $\t{\Theta}(\sqrt{T})$. Characterizing the set of
combining functions that induce hard bandit learning problems remains
an open problem.

Recall that in the bandit feedback model, the player only receives one
number as feedback on each round, namely, the value of
$f_t(X_{1:t})$. If the loss is a composite loss, we could also
consider a setting where the feedback consists of the single number
$\ell_t(X_t)$. Since the combining function $g$ is known to the player,
he could use the observed values $\ell_1(X_1),\ldots,\ell_t(X_t)$ to
calculate the value of $F_t(X_{1:t})$; this implies that this
alternative feedback model gives the player more information than the
strict bandit feedback model. However, it turns out that the
$\t{\Omega}(T^{2/3})$ lower bound holds even in this
alternative feedback model, so our analysis below assumes that the
player observes $\ell_t(X_t)$ on each round.

\paragraph{Organization.}

This paper is organized as follows. In \secref{sec:switching}, we
recall the analysis in \citet{DekelDiKoPe14} of the minimax regret of
the multi-armed bandit with switching costs. Components of this
analysis play a central role in the lower bounds against the composite
loss adversary. 
%We define the composite loss setting in \secref{sec:composite}. 
In \secref{sec:minAdversary} we prove a lower
bound on the minimax regret against the min-adversary in the bandit
feedback setting, and in \secref{sec:maxAdversary} we comment on how to
prove the same for the max-adversary.  
A proof that linear combining functions induce easy online learning problems is given in \secref{sec:linear}.
We conclude in \secref{sec:conc}.

\section{The Multi-Armed Bandit with Switching Costs}\label{sec:switching} 

In this section, we recall the analysis in \citet{DekelDiKoPe14},
which proves a $\t{\Omega}(T^{2/3})$ lower bound on the minimax
regret of the multi-armed bandit problem with switching costs. The new
results in the sections that follow build upon the constructions and
lemmas in \citet{DekelDiKoPe14}. For simplicity, we focus on the
$2$-armed bandit with switching costs, namely, we assume that $\Xcal =
\{0,1\}$ (see \citet{DekelDiKoPe14} for the analysis with arbitrary
$k$).

First, like many other lower bounds in online learning, we apply (the
easy direction of) Yao's minimax principle \citep{Yao77}, which states
that the regret of a randomized player against the worst-case loss
sequence is greater or equal to the minimax regret of an optimal
\emph{deterministic} player against a \emph{stochastic} loss
sequence. In other words, moving the randomness from the player to the
adversary can only make the problem easier for the player. Therefore,
it suffices to construct a \emph{stochastic} sequence of loss
functions\footnote{We use the notation $U_{i:j}$ as shorthand for the
  sequence $U_i,\ldots,U_j$ throughout.}, $F_{1:T}$, where each $F_t$
is a random oblivious loss function with a switching cost (as defined
in \eqref{eqn:addswitching}), such that
\begin{equation}\label{eqn:switchThm}
\E\left[ \sum_{t=1}^T F_t(X_{1:t}) ~-~ \min_{x \in \Xcal} \sum_{t=1}^T F_t(x,\ldots,x) \right]
~=~ \t{\Omega}(T^{2/3}) ~,
\end{equation}
for any deterministic player strategy.

We begin be defining a stochastic process $W_{0:T}$. Let $\xi_{1:T}$
be $T$ independent zero-mean Gaussian random variables with variance
$\sigma^2$, where $\sigma$ is specified below. Let $\rho:[T] \mapsto
\{0\} \cup [T]$ be a function that assigns each $t \in [T]$ with a
\emph{parent} $\rho(t)$. For now, we allow $\rho$ to be any function
that satisfies $\rho(t) < t$ for all $t$. Using $\xi_{1:T}$ and
$\rho$, we define
\begin{align}
W_0 &~=~ 0~~, \nonumber \\
\forall~t \in [T]~~~~ W_t &~=~ W_{\rho(t)} + \xi_t~~. \label{eqn:W}
\end{align}
Note that the constraint $\rho(t) < t$ guarantees that a recursive
application of $\rho$ always leads back to zero. The definition of the
parent function $\rho$ determines the behavior of the stochastic
processes. For example, setting $\rho(t) = 0$ implies that $W_t =
\xi_t$ for all $t$, so the stochastic process is simply a sequence of
i.i.d.~Gaussians. On the other hand, setting $\rho(t) = t-1$ results
in a Gaussian random walk.  Other definitions of $\rho$ can
create interesting dependencies between the variables. The specific
setting of $\rho$ that satisfies our needs is defined below.

Next, we explain how the stochastic process $W_{1:T}$ defines the
stochastic loss functions $F_{1:T}$. First, we randomly choose one of
the two actions to be the \emph{better action} by drawing an unbiased
Bernoulli $\chi$ ($\Pr(\chi=0) = \Pr(\chi=1)$). Then we 
let $\epsilon$ be a positive \emph{gap} parameter, whose value is
specified below, and we set 
\begin{equation}\label{eqn:zee}
\forall t\quad
Z_t(x) ~=~ W_t + \half - \epsilon \ind{x = \chi}~~.
\end{equation}
Note that $Z_t(\chi)$ is always smaller than $Z_t(1-\chi)$ by a
constant gap of $\epsilon$. Each function in the sequence $Z_{1:T}$
can take values on the entire real line, whereas we require bounded
loss functions. To resolve this, we confine the values of $Z_{1:T}$ to
the interval $[0,1]$ by applying a clipping operation,
\begin{equation}\label{eqn:clipping}
\forall\,t \quad L_t(x) ~=~ \clip(Z_t(x))~~, \quad \text{where}~~ \clip(\alpha) ~=~ \min\{\max\{ \alpha, 0 \}, 1\}~~.
\end{equation}
The sequence $L_{1:T}$ should be thought of as a stochastic oblivious
loss sequence.  Finally, as in \eqref{eqn:addswitching}, we add a
switching cost and define the sequence of loss functions
$$
F_t(x_{1:T}) ~=~ L_t(x_t) + 1\!\!1_{x' \neq x} ~~.
$$ 
It remains to specify the parent function $\rho$, the standard
deviation $\sigma$, and the gap $\epsilon$. With the right settings,
we can prove that $F_{1:T}$ is a stochastic loss sequence that satisfies 
\eqref{eqn:switchThm}.

We take a closer look at the parent function $\rho$. First, we define the \emph{ancestors}
of round $t$, denoted by~\,$\anc(t)$, to be the set of positive
indices that are encountered when $\rho$ is applied recursively to
$t$. Formally, $\anc(t)$ is defined recursively as
\begin{align} \label{eqn:ancestor}
  \anc(0) &~=~ \{\} \nonumber \\
\forall~t \quad \anc(t) &~=~  \anc\big(\rho(t)\big)~\cup~\{\rho(t)\}~~. 
\end{align}
Using this definition, the \emph{depth} of $\rho$ is defined as the size of the
largest set of ancestors, $\depth(\rho) = \max_{t \in [T]}
|\anc(t)|$.  The depth is a key property of $\rho$ and the value of $\depth(\rho)$ characterizes the extremal
values of $W_{1:T}$: by definition, there exists a
round $t$ such that $W_t$ is the sum of $\depth(\rho)$ independent
Gaussians, so the typical value of $|W_t|$ is bounded by $\sigma
\sqrt{\depth(\rho)}$.  More precisely, Lemma 1 in
\citet{DekelDiKoPe14} states that
\begin{align} \label{eqn:depthbound}
\forall ~ \delta\in(0,1) \qquad	\Pr\lr{\max_{t \in [T]} \abs{W_t}
	\le \sig \sqrt{2\depth(\rho) \log\tfrac{T}{\delta}} }  
	~\ge~ 1-\delta~.
\end{align}
The clipping operation defined in \eqref{eqn:clipping} ensures that
the loss is bounded, but the analysis requires that the unclipped
sequence $Z_{1:t}$ already be bounded in $[0,1]$ with high probability. This
implies that we should choose 
\begin{equation}\label{eqn:sigma}
\sigma ~\sim~ \big(\depth(\rho) \log\left(\tfrac{T}{\delta}\right) \big)^{-1/2} ~~.
\end{equation}

Another important property of $\rho$ is its \emph{width}. First, 
define the \emph{cut} on round $t$ as
$$
\cut(t) ~=~ \set{s \in [T] ~:~ \rho(s) < t \le s}~.
$$ 
In words, the cut on round $t$ is the set of rounds that are separated from their parent by $t$.  The
\emph{width} of $\rho$ is then defined as the size of the largest cut, $\width(\rho) = \max_{t \in
  [T]}|\cut(t)|$. 

The analysis in \citet{DekelDiKoPe14} characterizes the player's
ability to statistically estimate the value of $\chi$ (namely, to
uncover the identity of the better action) as a function of the number
of switches he performs. Each time the player switches actions, he has
an opportunity to collect statistical information on the identity of
$\chi$. The amount of information revealed to the player with each
switch is controlled by the depth and width of $\rho$ and the values
of $\epsilon$ and $\sigma$. Formally, define the conditional
probability measures
\begin{equation}\label{eqn:Q0Q1}
\Q_0(\cdot) ~=~ \Pr(\cdot \mid \chi=0) 
\qquad \text{and} \qquad 
\Q_1(\cdot) ~=~ \Pr(\cdot \mid \chi=1) ~~.
\end{equation}
In words, $\Q_0$ is the conditional probability when action $0$ is
better and $Q_1$ is the conditional probability when action $1$ is
better. Also, let $\Fcal$ be the $\sigma$-algebra generated by the
player's observations throughout the game, $L_1(X_1),\ldots,L_T(X_T)$.
Since the player's actions are a deterministic function of the loss
values that he observes, his sequence of actions is measurable by
$\Fcal$. The \emph{total variation} distance between $\Q_0$ and $\Q_1$
on $\Fcal$ is defined as
$$
d_\mathrm{TV}^{\Fcal}(\Q_0, Q_1) ~=~
\sup_{A \in \Fcal} \big|\Q_0(A) - \Q_1(A) \big| ~~.
$$
\citet{DekelDiKoPe14} proves the following bound on $d_\mathrm{TV}^{\Fcal}(\Q_0, Q_1)$.
\begin{lemma}\label{lma:dtv}
Let $F_{1:T}$ be the stochastic loss sequence defined above by the
parent function $\rho$, with variance $\sigma^2$ and gap
$\epsilon$. Fix a deterministic player and let $M$ be the number of
switches he performs as he plays the online game. Then,
$$
d_\mathrm{TV}^{\Fcal}(\Q_0, Q_1) ~\le~ \frac{\eps}{\sigma}  \sqrt{\width(\rho) \, \E[M]} ~~,
$$ 
where $Q_0$ and $Q_1$ are as defined in 
\eqref{eqn:Q0Q1}.
\end{lemma}
Intuitively, the lemma states that if $\E[M]$ is asymptotically
smaller than $\frac{\sigma^2}{\epsilon^2 \width(\rho)}$ then any
$\Fcal$-measurable event (e.g., the event that $X_{10} = 1$ or the
event that the player switches actions on round $20$) is almost
equally likely to occur, whether $\chi=0$ or $\chi=1$. In other words,
if the player doesn't switch often enough, then he certainly cannot
identify the better arm.
 
Our goal is to build a stochastic loss sequence that forces the player
to perform many switches, and \lemref{lma:dtv} tells us that we must
choose a parent function $\rho$ that has a small width. Additionally,
setting the variance $\sigma^2$ according to \eqref{eqn:sigma} also
implies that we want $\rho$ to have a small depth.
\citet{DekelDiKoPe14} defines the parent function $\rho(t) = t
- \gcd(t, 2^T)$ (where $\gcd(\alpha,\beta)$ is the greatest
common divisor of $\alpha$ and $\beta$). Put another way, $\rho$ takes
the number $t$, finds its binary representation, identifies the least
significant bit that equals $1$, and flips that bit to zero. It them
proves that $\depth(\rho) = \Theta(\log T)$ and $\width(\rho) =
\Theta(\log T)$.

The lower bound on the minimax regret of the multi-armed bandit with
switching costs is obtained by setting $\eps = \Theta(T^{-1/3}/
\log T)$. If the expected number of switches is small, namely $\E[M]
\leq T^{2/3}/\log^2 T$, then \lemref{lma:dtv} implies that the
player cannot identify the better action. From there, it is
straightforward to show that the player has a positive probability of
choosing the worse action on each round, resulting in a regret of $R =
\Theta(\epsilon T)$. Plugging in our choice of $\epsilon$ proves that
$R = \t{\Omega}(T^{2/3})$. On the other hand, if the number of
switches is large, namely, $\E[M] > T^{2/3}/\log^2 T$, then
the regret is $\Omega(T^{2/3})$ directly due to the switching cost.

Many of the key constructions and ideas behind this proof are reused below.

\section{The Min Adversary with Bandit Feedback is Hard} \label{sec:minAdversary}

In this section, we lower bound the minimax regret against the
min-adversary in the feedback model where the player only observes a
single number, $\ell_t(X_t)$, at the end of round $t$. The full proof
is rather technical, so we begin with a high level proof sketch.  As
in \secref{sec:switching}, Yao's minimax principle once again reduces
our problem to one of finding a stochastic loss sequence $L_{1}, \ldots, L_{T}$
that forces all deterministic algorithms to incur a regret of $\t{
\Omega}(T^{2/3})$. The main idea is to repeat the construction
presented in \secref{sec:switching} by simulating a switching cost
using the min combining function.

We start with a stochastic process that is defined by a parent
function $\rho$, similar to the sequence $W_{1:T}$ defined in
\secref{sec:switching} (although we require a different parent
function than the one defined there). Again, we draw a
Bernoulli $\chi$ that determines the better of the two possible
actions, we choose a \emph{gap} parameter $\epsilon$, and we define
the sequence of functions $Z_{1:T}$, as in \eqref{eqn:zee}. This
sequence has the important property that, in the bandit feedback
model, it reveals information on the value of $\chi$ only when the
player switches actions.

Next, we identify triplets of rounds, $(t-1, t, t+1)$, where
$|W_{t-1} - W_t| \leq \tau$ ($\tau$ is a \emph{tolerance} parameter,
chosen so that $\tau \gg \epsilon$) and some other technical
properties hold.  Then, we simulate a switching cost on round $t$ by
adding a pair of \emph{spikes} to the loss values of the two actions,
one on rounds $t-1$ and one on round $t$. We choose a \emph{spike
  size} $\eta$ (such that $\eta \gg \tau$), we draw an unbiased
Bernoulli $\Lambda_t$, and we set
$$
L_{t-1}(x) ~=~ \clip\big(Z_{t-1}(x) + \eta \ind{x \neq \Lambda_t}\big) 
\qquad \text{and} \qquad 
L_{t}(x) ~=~ \clip\big(Z_t(x) + \eta \ind{x = \Lambda_t} \big) ~~,
$$ 
where $\clip()$ is defined in \eqref{eqn:clipping}. In words, with
probability $\half$ we add a spike of size $\eta$ to the loss of
action $0$ on round $t-1$ and to the loss of action $1$ on round $t$,
and with probability $\half$ we do the opposite.

Finally, we define the loss on round $t$ using the min combining
function
\begin{equation}\label{eqn:minLoss}
F_t(x_{1:t}) ~=~ \min\big(L_{t-1}(x_{t-1}), L_t(x_t)\big) ~~.
\end{equation}

We can now demonstrate how the added spikes simulate a switching
cost on the order of $\eta$.
Say that the player switches actions on round $t$, namely, $X_t \neq
X_{t-1}$.  Since $\Lambda_t$ is an independent unbiased Bernoulli, it holds
that $X_t = \Lambda_t$ with probability $\half$. If $X_t = \Lambda_t$, then the
player encounters both of the spikes: $L_t(X_t) = Z_t(X_t) + \eta$ and
$L_{t-1}(X_{t-1}) = Z_{t-1}(X_{t-1}) + \eta$. Recall that $|Z_{t-1}(0) - Z_{t-1}(1)|
\leq \epsilon$ and $|Z_{t-1}(x) - Z_t(x)| \leq \tau$, so
\begin{equation}\label{eqn:int1}
F_{t}(X_{1:t}) ~\in~ \big[ Z_t(0) + \eta - (\epsilon + \tau),~ Z_t(0) + \eta + (\epsilon + \tau) \big] ~~.
\end{equation}
On the other hand, if the player does not switch actions on round
$t$, his loss then satisfies
\begin{equation}\label{eqn:int2}
F_{t}(X_{1:t}) ~\in~ \big[ Z_t(0) - (\epsilon + \tau),~ Z_t(0) + (\epsilon + \tau) \big] ~~.
\end{equation}
Comparing the intervals in \eqref{eqn:int1} and \eqref{eqn:int2}, and
recalling that $\eta \gg (\epsilon + \tau)$, we conclude that, with
probability $\half$, the switch caused the player's loss to increase by
$\eta$. This is the general scheme by which we simulate a switching
cost using the min combining function.

There are a several delicate issues that were overlooked in the
simplistic proof sketch, and we deal with then below.

\subsection{The Stochastic Loss Sequence}

We formally describe the stochastic loss sequence used to prove our
lower bound. In \secref{sec:switching}, we required a deterministic
parent function $\rho$ with depth $\depth(\rho)$ and width
$\width(\rho)$ that scale logarithmically with $T$. To lower-bound the
minimax regret against the min adversary, we need a \emph{random}
parent function for which $\depth(\rho)$ and $\width(\rho)$ are both logarithmic with high
probability, and such that $\rho(t)=t-1$ with probability at least $\half$ for all $t$. 
The following lemma proves that such a random parent function exists.

\begin{lemma} \label{lma:rho} 
For any time horizon $T$, there 
exists a random function $\rho:[T] \mapsto \{0\} \cup [T]$ with $\rho(t) < t$ 
for all $t \in [T]$ such that 
\begin{itemize}
\item $\forall\,t\quad\Pr\big(\rho(t) = t-1 \mid \rho(1),\ldots,\rho(t-1)\big) \geq \half$;
\item $\width(\rho) \leq \log T + 1$ with probability $1$;
\item $\depth(\rho) = O(\log T)$ with probability $1-O(T^{-1})$.
\end{itemize}
\end{lemma}

\begin{proof}%[Proof of \lemref{lma:rho}]
We begin with the deterministic parent function used in
\secref{sec:switching}, denoted here by $\tilde \rho$, and defined as
$\tilde \rho(t) = t - \gcd(t, 2^T)$. Additionally, draw independent
unbiased Bernoullis $B_{1:T}$. We now define the random function
$\rho$.  If $B_t=0$ then set $\rho(t) = t-1$. To define the remaining
values of $\rho$, rename the ordered sequence $(t : B_t = 1)$ as
$(U_1,U_2,\ldots)$ and also set $U_0 = 0$.  If $B_t=1$, let $k$ be
such that $t = U_k$ and set $\rho(t) = U_{\tilde \rho(k)}$. This
concludes our construction, and we move on to prove that it satisfies
the desired properties.

The probability that $\rho(t)=t-1$ is at least $\half$,
since this occurs whenever the unbiased bit $B_t$ equals zero.
\citet{DekelDiKoPe14} proves that the width of $\tilde\rho$ is bounded
by $\log T + 1$, and the width of $\rho$ never exceeds this bound.
\citet{DekelDiKoPe14} also proves that the depth of $\tilde \rho$ is
bounded by $\log T + 1$. The depth difference $\depth(\rho)-\depth(\tilde \rho)$
is at most $\max_k (U_k-U_{k-1})$ by construction. A union bound implies  that the probability  that this maximum  exceeds $\ell=2\log T$ is at most $T \cdot 2^{-\ell}=T^{-1}$.
Thus $\Pr\bigl(\depth(\rho) \ge 4\log T \bigr) \le T^{-1}$. 
\end{proof}

Let $\rho$ be a random parent function, as described above, and use
this $\rho$ to define the loss sequence $Z_{1:T}$, as outlined in
\secref{sec:switching}. Namely, draw independent zero-mean Gaussians
$\xi_{1:T}$ with variance $\sigma^2$.  Using $\rho$ and $\xi_{1:T}$,
define the stochastic process $W_{1:T}$ as specified in \eqref{eqn:W}.
Finally, choose the better arm by drawing an unbiased Bernoulli
$\chi$, set a gap parameter $\epsilon$, and use $W_{1:T}$, $\chi$, and
$\epsilon$ to define the loss sequence $Z_{1:T}$, as in
\eqref{eqn:zee}.

Next, we augment the loss sequence $Z_{1:T}$ in a way that simulates a
switching cost.  For all $2 \leq t \leq T-2$, let $E_t$ be the
following event:
\begin{align} \label{eq:Et}
	E_{t} = \big\{\,
		|W_{t-1} - W_t| \leq \tau
		\quad\text{and}\quad
		W_{t+1} < W_t - \tau
		\quad\text{and}\quad
		W_{t+2} < W_{t+1} - \tau
	\,\big\} ~,
\end{align}
where $\tau$ is a tolerance parameter defined below. In other
words, $E_t$ occurs if the stochastic process $W_{1:T}$ remains rather
flat between rounds $t-1$ and $t$, and then drops on rounds $t+1$ and $t+2$.  We
simulate a switching cost on round $t$ if and only if $E_t$
occurs. 

We simulate the switching cost by adding pairs of \emph{spikes}, one
to the loss of each action, one on round $t-1$ and one on round
$t$. Each spike has an orientation: it either penalizes a switch from
action $0$ to action $1$, or a switch from action $1$ to action $0$.
The orientation of each spike is chosen randomly, as follows. We draw
independent unbiased Bernoullis $\Lambda_{2:T-1}$; if a spike is added on round
$t$, it penalizes a switch from action $X_{t-1} = 1-\Lambda_t$ to action $X_t =
\Lambda_t$. Formally, define
$$
S_t(x) ~=~ 
\begin{cases}
~\eta &~~ \text{if}~ (E_t ~\land~ x=\Lambda_t) \lor (E_{t+1} ~\land~ x \neq \Lambda_{t+1}) \\
~0 &~~ \text{otherwise}
\end{cases}~~,
$$ 
where $\eta$ is a \emph{spike size} parameter (defined below).
Finally, define $L_t(x) = \clip\big(Z_t(x) + S_t(x)\big)$. This
defines the sequence of oblivious functions. The min adversary uses
these functions to define the loss functions $F_{1:T}$, as in
\eqref{eqn:minLoss}.

In the rest of the section we prove that the regret of any deterministic player against the loss
sequence $F_{1:T}$ is $\t{\Omega}(T^{2/3})$.  Formally, we prove the following theorem.

\begin{theorem}\label{thm:main}
Let $F_{1:T}$ be the stochastic sequence of loss functions defined above.
Then, the expected regret (as defined in
\eqref{eq:regret}) of any deterministic player against this sequence
is~$\t{\Omega}(T^{2/3})$.
\end{theorem}

\subsection{Analysis}

For simplicity, we allow ourselves to neglect the clipping operator
used in the definition of the loss sequence, an we simply assume that
$L_t(x) = Z_t(x) + S_t(x)$. The additional steps required to
reintroduce the clipping operator are irrelevant to the current
analysis and can be copied from \citet{DekelDiKoPe14}. 

Fix a deterministic algorithm and let $X_{1},\ldots,X_{T}$ denote the random sequence of actions it chooses upon the stochastic loss functions $F_{1:T}$. 
We define the algorithm's instantaneous (per-round) regret as
\begin{equation}\label{eqn:instRegret}
\forall\,t\quad R_t ~=~ \min\big( L_{t-1}(X_{t-1}), L_t(X_t) \big) -  \min\big( L_{t-1}(\chi), L_t(\chi) \big) ~~,
\end{equation}
and note that our goal is to lower-bound $\E[R] = \sum_{t=1}^T \E[R_t]$. 

The main technical difficulty of our analysis is getting a handle on
the player's ability to identify the occurrence of $E_t$. If the
player could confidently identify $E_t$ on round $t-1$, he could avoid
switching on round $t$. If the player could identify $E_t$ on round
$t$ or $t+1$, he could safely switch on round $t+1$ or $t+2$, as
$E_{t}$ cannot co-occur with either $E_{t+1}$ or $E_{t+2}$. To this
end, we define the following sequence of random variables,
\begin{equation}\label{eqn:tR}
\forall\,t\quad \tR_t ~=~ \min\big( L_{t-1}(X_{t-1}), L_t(X_{t-1}) \big) -  \min\big( L_{t-1}(\chi), L_t(\chi) \big) ~~.
\end{equation}
The variable $\tR_t$ is similar to $R_t$, except that $L_t$ is
evaluated on the previous action $X_{t-1}$ rather than the current
action $X_t$. We think of $\tR_t$ as the instantaneous regret of a
player that decides beforehand (before observing the value of
$L_{t-1}(X_{t-1})$) not to switch on round $t$. It turns out that
$\tR_t$ is much easier to analyze, since the player's decision to
switch becomes independent of the occurrence of $E_t$. Specifically, we
use $\tR_t$ to decompose the expected regret as
$$
\E[R_t] ~=~ \E[R_t - \tR_t] ~+~ \E[\tR_t] ~~.
$$

We begin the analysis by clarifying the requirement that the event $E_{t}$ only
occurs if $W_{t+1} \leq W_t - \tau$. This requirement serves two
separate roles: first, it prevents $E_t$ and $E_{t+1}$ from
co-occurring and thus prevents overlapping spikes; second, this
requirement prevents $E_{t-1}$ from contributing to the player's loss
on round $t$. This latter property is used throughout our analysis and is
formalized in the following lemma.

\begin{lemma} \label{lem:E2}
If $E_{t-1}$ occurs then 
%$F_{t}(x_{1:t}) = Z_{t}(x_{t})$ for any sequence of actions $x_{1:t}$.
$\tR_t = Z_t(X_{t-1}) - Z_t(\chi)$ and $R_t -\tR_t = Z_t(X_{t}) - Z_t(X_{t-1})$.
%\tk{the lemma is false - we have to fix the construction?}
\end{lemma}

In particular, the lemma shows that the occurrence of $E_{t-1}$ 
cannot make $F_t(X_{1:t})$ be less than $F_t(\chi,\ldots,\chi)$. 
This may not be obvious at first glance: the
occurrence of $E_{t-1}$ contributes a spike on round $t$ and if that
spike is added to $\chi$ (the better action), one might imagine that
this spike could contribute to $F_t(\chi,\ldots,\chi)$.

\begin{proof}[Proof of \lemref{lem:E2}]
Note that the occurrence of $E_{t-1}$ implies that $W_t \leq W_{t-1} -
\tau$ and that $W_{t+1} \leq W_{t} - \tau$, which means that a spike
is not added on round $t$. Therefore, $L_t(x_t) = Z_t(x_t)$ for any
$x_t$ and $\min\big(L_{t-1}(x_{t-1}), L_{t}(x_{t})\big) =
Z_{t}(x_{t})$ for any $x_{t-1}$ and $x_t$. The first claim follows
from two applications of this observation: once with $x_{t-1} = x_{t} = X_{t-1}$
and once with $x_{t-1} = x_{t} = \chi$. The second claim is obtained by
setting $x_{t-1} = X_{t-1}, x_{t} = X_{t}$.
\end{proof}

%\begin{remark} \label{remark:E2}
%If $E_{t-1}$ occurs then $F_t(x_{1:t}) = Z_t(x_t)$. Therefore, the
%occurrence of $E_{t-1}$ cannot make $F_t(X_{1:t})$ be less than
%$F_t(\chi,\ldots,\chi)$. This may not be obvious at first glance: the
%occurrence of $E_{t-1}$ contributes a spike on round $t$ and if that
%spike is added to $\chi$ (the better action), one might imagine that
%this spike could contribute to $F_t(\chi,\ldots,\chi)$. However, note
%that $E_{t-1}$ implies that $W_t \leq W_{t-1} - \tau$, and therefore
%$\min\big(L_{t-1}(x_{t-1}), L_{t}(x_{t})\big) = L_{t}(x_{t})$ for any
%$x_{t-1}$ and $x_t$. Moreover, $E_{t-1}$ also implies that $\lnot
%E_{t}$, so it actually holds that $L_t(x_t) = Z_t(x_t)$ for any $x_t$.
%\end{remark}

%%%Let $\Gamma$ denote the indicator of the event that  $|L_t(X_t)-L_{t_1}| \le 2\sigma\sqrt{\log{T}}$ for all $t \le T$; note this has probability  at least $1-O(1/T)$. Since %$M \le T$, we deduce that $\E[M \Gamma]
%%%\geq \frac{T^{2/3}}{\log{T}}$. Conditioning on $\Gamma$ will be useful at several points below.

It is convenient to modify the algorithm and fix $X_s=X_t$ for all $s>t$ if $|L_t(X_t)-L_{\rho(t)}(X_{\rho(t)})| \ge 4\sigma\sqrt{\log{T}}$. Note that this event
has probability $O(T^{-4})$ for each $t$, so  the modification has a negligible effect on the regret.
%%\tk{we can assume this without modifying the alg, looks better that way}
%%\tk{state as a lemma, give a proof}
Recall that $\E[R_t] = \E[R_t - \tR_t] + \E[\tR_t]$; we first claim that $\E[\tR_t]$ is non-negative.

\begin{lemma}\label{lma:noSwitch}
For any $1<t<T$, let $\tR_t$ be as defined in \eqref{eqn:tR}. Then, it holds that~$\E[\tR_t \mid X_{t-1} = \chi] = 0$ and $\E[\tR_t \mid X_{t-1} \neq \chi] = \epsilon$.
\end{lemma}

Next, we turn to lower bounding $\E[R_t - \tR_t]$. 

\begin{lemma}\label{lma:switch}
For any $1<t<T$, let $R_t$ be the player's instantaneous regret, as defined in \eqref{eqn:instRegret}, and let $\tR_t$ be as defined in \eqref{eqn:tR}. 
Then $\E[R_t - \tR_t] = \Pr(X_t \neq X_{t-1}) \cdot \Omega(\eta \tau / \sig)$, provided that $\tau = o(\eta)$ and that $\eps = o(\eta \tau / \sig)$.
\end{lemma}

The proofs of both lemmas are deferred to \secref{sub:technical} below.
We can now prove our main theorem.

\begin{proof}[Proof of \thmref{thm:main}]
We prove the theorem by distinguishing between two cases, based
on the expected number of switches performed by the player. More
specifically, let $M$ be the number of switches performed by the
player throughout the game.

First, assume that $\E[M] \geq T^{2/3}/\log^2{T}$.
Summing the lower-bounds in \lemref{lma:noSwitch} and \lemref{lma:switch} over all $t$ gives
$$
\E[R] ~\ge~ \sum_{t=1}^{T} \Pr(X_t \neq X_{t-1}) \cdot \Omega(\eta \tau / \sig) 
~=~ \Omega(\eta \tau / \sig) \cdot \E[M] ~~.
$$
Setting $\eta = \log^{-2} T$, $\sig = \log^{-1} T$, $\tau = \log^{-5} T$, $\eps = T^{-1/3}/\log T$ (note that all of the constraints on these values specified in \lemref{lma:switch} are met)
and plugging in our assumption that $\E[M] \geq T^{2/3}/\log^2{T}$ gives the lower bound
$$
\E[R] ~=~ \Omega \left( \frac{T^{2/3}}{\log^6 T} \right) ~~.
$$ 

Next, we assume that $\E[M] < T^{2/3}/\log^2{T}$.
%This analysis is very reminiscent
%of the one in \citet{DekelDiKoPe14}, so we only sketch the proof
%below. One difference between the setting in \citet{DekelDiKoPe14} and
%the one considered here is that the parent function $\rho$ is
%stochastic. 
For any concrete instance of $\rho$, define the conditional
probability measures
$$
\Q^\rho_0(\cdot) ~=~ \Pr(\cdot \mid \chi=0, \rho) \qquad \text{and} \qquad 
\Q^\rho_1(\cdot) ~=~ \Pr(\cdot \mid \chi=1, \rho) ~~.
$$ 
We can apply \lemref{lma:dtv} for any concrete instance of $\rho$ and get
$$
d_\mathrm{TV}^{\Fcal}(\Q^\rho_0, Q^\rho_1) ~\le~ \frac{\eps}{\sigma}  \sqrt{\width(\rho) \, \E[M \mid \rho]} ~~.
$$ 
Taking expectation on both sides of the above, we get
$$
d_\mathrm{TV}^{\Fcal}(\Q_0, Q_1) ~\leq~ \E\big[ d_\mathrm{TV}^{\Fcal}(\Q^\rho_0, Q^\rho_1) \big] ~\leq~
\frac{\eps}{\sigma}  \E\big[ \sqrt{\width(\rho) \, \E[M \mid \rho]} \big] ~\leq~ 
\frac{\eps}{\sigma} \sqrt{\E[\width(\rho)] \, \E[M]}~~,
$$ 
Where the inequality on the left is due to Jensen's inequality, and
the inequality on the right is due to an application of the
Cauchy-Schwartz inequality.  Plugging in $\eps =
T^{-1/3}/\log T$, $\sig = \log^{-1} T$,
$\E[\width(\rho)] = \Theta(\log{T})$ and $\E[M] =
O( T^{2/3}/\log^2 T )$, we conclude that
\begin{equation}\label{eqn:overallTV}
d_\mathrm{TV}^{\Fcal}(\Q_0, Q_1) ~=~ o(1)~~.
\end{equation}

Again, we decompose $\E[R_t] = \E[\tR_t] + \E[R_t - \tR_t]$, but this
time we use the fact that \lemref{lma:switch} implies $\E[R_t - \tR_t]
\geq 0$, and we focus on lower-bounding $\E[\tR_t]$. We decompose
\begin{equation}\label{eqn:AA}
\E[\tR_t] ~=~ \Pr(X_{t=1} = \chi) ~ \E[\tR_t \mid X_{t=1} = \chi] ~+~  \Pr(X_{t=1} \neq \chi) ~ \E[\tR_t \mid X_{t=1} \neq \chi] ~~.  
\end{equation}
The first summand on the right-hand side above trivially equals
zero. \lemref{lma:noSwitch} proves that $\E[\tR_t \mid X_{t=1} \neq \chi]
= \epsilon$. We use
\eqref{eqn:overallTV} to bound
\begin{align*}
\Pr(X_{t-1} \neq \chi) &~=~ \half\, \Pr(X_{t-1} = 0 \mid \chi = 1) + \half\, \Pr(X_{t-1} = 1 \mid \chi = 0) \\
&~\geq~ \half\, \Pr(X_{t-1} = 0 \mid \chi = 1) + \half\, \big( \Pr(X_{t-1} = 1 \mid \chi = 1) - o(1) \big) \\
&~=~ \half - o(1) ~~.
\end{align*}
Plugging everything back into \eqref{eqn:AA} gives $\E[\tR_t] = \Theta(\epsilon)$. We conclude that 
$$
\E[R] ~\geq~ \sum_{t=1}^T \E[\tR_t] ~=~ \Theta(T \epsilon) ~~.
$$
Recalling that $\eps = T^{-1/3}/\log T$ concludes the analysis.
\end{proof}

\subsection{Technical Proofs} \label{sub:technical}

We now provide the proofs of the technical lemmas stated above.
We begin with \lemref{lma:noSwitch}.

\begin{proof}[Proof of \lemref{lma:noSwitch}]
If $X_{t-1} = \chi$ then $\tR_t = 0$ trivially. Assume henceforth that
$X_{t-1} \neq \chi$. If $E_{t-1}$ occurs then \lemref{lem:E2}
guarantees that $\tR_t = Z_t(X_{t-1}) - Z_t(\chi)$, which equals
$\epsilon$ by the definition of $Z_t$. If $\lnot E_{t-1}$ and $\lnot
E_t$ then
$$
\tR_t ~=~ \min\big(Z_{t-1}(X_{t-1}),Z_t(X_{t-1})\big) - \min\big(Z_{t-1}(\chi),Z_t(\chi)\big)~~,
$$ 
which, again, equals $\epsilon$.  If $E_t$ occurs then the loss
depends on whether $W_{t-1} \geq W_t$ and on the value of
$\Lambda_t$. We can first focus on the case where $W_{t-1} \geq
W_t$. If $\Lambda_t \neq X_{t-1}$ then the assumption that $\eta \gg
\tau$ implies that $\min(L_{t-1}(\chi), L_t(\chi)) = Z_{t-1}(\chi)$
and $\min(L_{t-1}(X_{t-1}), L_t(X_{t-1})) = Z_t(X_{t-1})$, and
therefore
\begin{equation}\label{eqn:nos1}
\tR_t ~=~ Z_t(X_{t-1}) - Z_{t-1}(\chi) ~=~ \epsilon - |W_{t-1} - W_t| ~~,
\end{equation}
which could be negative. On the other hand, if $\Lambda_t = X_{t-1}$, then
$\min(L_{t-1}(\chi), L_t(\chi)) = Z_{t}(\chi)$ and $\min(L_{t-1}(X_{t-1}), L_t(X_{t-1}))
= Z_{t-1}(X_{t-1})$, and therefore
\begin{equation}\label{eqn:nos2}
\tR_t ~=~ Z_{t-1}(X_{t-1}) - Z_t(\chi) ~=~ \epsilon + |W_{t-1} - W_t| ~~.
\end{equation}
Now note that $\Lambda_t$ is an unbiased Bernoulli that is independent
of $X_{t-1}$ (this argument would have failed had we directly analyzed
$R_t$ instead of $\tR_t$). Therefore, the possibility of having a
negative regret in \eqref{eqn:nos1} is offset by the equally probable
possibility of a positive regret in \eqref{eqn:nos2}. In other words,
$$
\E\big[\tR_t \,\big|\, X_{t-1} \neq \chi, W_{t-1} \geq W_t, E_t\big] 
~=~ \half\big(\epsilon - |W_{t-1} - W_t|\big) + \half\big(\epsilon + |W_{t-1} - W_t|\big) 
~=~ \epsilon~~.
$$
The same calculation applies when $W_{t-1} < W_t$. Overall, we have shown that $\E[\tR_t \mid X_{t-1} \neq \chi] = \epsilon$.
%and therefore $\E[\tR_t] \geq 0$. 
\end{proof}

Next, we prove \lemref{lma:switch}.

\begin{proof}[Proof of \lemref{lma:switch}]
Since $R_t$ and $\tR_t$ only differ when $X_t \neq X_{t-1}$, we have that
$$
\E[R_t - \tR_t] ~=~ \Pr(X_t \neq X_{t-1}) ~ \E[R_t - \tR_t \mid X_t \neq X_{t-1}]~~,
$$
so it remains to prove that  $\E[R_t - \tR_t \mid X_t \neq X_{t-1}] = \Omega(\eta \tau / \sigma)$.
We deal with two cases, depending on the occurrence of $E_t$, and write
\begin{align}
\E[R_t - \tR_t \mid X_t \neq X_{t-1}] ~=~ &\Pr(\lnot E_t \mid X_t \neq X_{t-1})~ \E[R_t - \tR_t \mid X_t \neq X_{t-1}, \lnot E_t] \nonumber\\
&~+~ \Pr(E_t \mid X_t \neq X_{t-1})~ \E[R_t - \tR_t \mid X_t \neq X_{t-1}, E_t]~~. \label{eqn:decomp1}
\end{align}

We begin by lower-bounding the first case, where $\lnot E_t$.  If
$E_{t-1}$ occurs, then \lemref{lem:E2} guarantees that $R_t -
\tR_t = Z_t(X_{t}) - Z_t(X_{t-1})$, which is at least $- \epsilon$.
Otherwise, if neither $E_{t-1}$ or $E_t$ occur, then again $R_t -
\tR_t \geq - \epsilon$. We upper-bound $\Pr(\lnot E_t \mid X_t \neq X_{t-1})
\leq 1$ and get that
\begin{equation} \label{eqn:negligibleLowerBnd}
\Pr(\lnot E_t \mid X_t \neq X_{t-1})~ \E[R_t - \tR_t \mid X_t \neq X_{t-1}, \lnot E_t] ~\geq~ - \epsilon ~~.
\end{equation}

Next, we lower-bound the second case, where $E_t$. \lemref{lma:ProbE}
below lower-bounds $\Pr(E_t \mid X_t \neq X_{t-1}) = \Omega(\tau /
\sigma)$. \lemref{lma:switchingPenalty} below lower-bounds $\E[R_t -
  \tR_t \mid X_t \neq X_{t-1}, E_t] \geq \eta/3 - \tau$ for $T$
sufficiently large. Recalling the assumption that $\eta \gg \tau$, we conclude that
$$
\Pr(E_t \mid X_t \neq X_{t-1})~ \E[R_t - \tR_t \mid X_t \neq X_{t-1}, E_t] ~=~ \Omega\left( \frac{\eta \tau}{\sigma} \right) ~~.
$$ 
\eqref{eqn:negligibleLowerBnd} can be neglected since
$\eta \tau / \sigma \gg \epsilon$, and this concludes the proof.
\end{proof}

\begin{lemma}\label{lma:ProbE} 
Suppose $\eta, \tau \leq \sigma/\log T$. For all $t > 1$ it holds that
$\Pr(E_t \mid X_t \neq X_{t-1}) ~=~ \Omega(\tau / \sig)$.
\end{lemma}

\begin{proof}
By our earlier modification of the algorithm,
we assume that 
\begin{equation}\label{eq-assumption-lemma6}
	|L_{s}(X_{s})- L_{\rho(s)}(X_{\rho(s)}) | 
	~\le~ 4\sigma\sqrt{\log{T}} \mbox{ for } s\in \{t-2, t-1\} 
\end{equation}
(which occurs with probability at least $1 - O(T^{-4})$). Otherwise, the event $X_t \neq X_{t-1}$ would never occur due to our modification of the algorithm and the statement is irrelevant. 

In order to prove the lemma, we verify a stronger statement that  $\Pr(E_t \mid \Fcal_{t-1}) =
\Omega(\tau / \sigma)$, where $\Fcal_{t-1}$ is the $\sigma$-field
generated by the player's observations up to round $t-1$ (note that
$X_t$ is $\Fcal_{t-1}$-measurable). Let $f_1(\ell_1, \ldots, \ell_{t-1})$ be the conditional density of $(L_1(X_1), \ldots, L_{t-1}(X_{t-1}))$ given $E_{t-1}$, and let $f_2(\ell_1, \ldots, \ell_{t-1})$ be the conditional density of $(L_1(X_1), \ldots, L_{t-1}(X_{t-1}))$ given $E_{t-1}^c$. We get that
$$
	\min \frac{f_2(\ell_1, \ldots, \ell_{t-1}) \Pr(E_{t-1}^c)}{f_1(\ell_1, \ldots, \ell_{t-1}) \Pr(E_{t-1})} 
	~\geq~ \min_{x\leq 4 \sigma \sqrt{\log T}} \Big(\frac{\mathrm{e}^{x^2/\sigma^2}}{\mathrm{e}^{(x+\eta)^2/\sigma^2}}\Big)^2 
	~=~ 1-o(1)\,,
$$
where the the first minimum is over all sequences that are compatible with \eqref{eq-assumption-lemma6}, and the last inequality follows from the assumption that $\eta \leq \sigma/\log T$.
Hence,  we have 
\begin{equation} \label{eq-lemma-6-1}
	\Pr(E_{t-1}^c \mid \mathcal F_{t-1}) \geq 1/2 + o(1)\,.
\end{equation} Further, we see that
\begin{align} \label{eq-lemma-6-2}
	&\frac{\Pr(E_{t-1}^c, \rho(t) = t-1, \rho(t+1) = t, \xi_t \geq -\tau\mid \mathcal F_{t-1})}{\Pr(E_{t-1}^c \mid \mathcal F_{t-1})} \nonumber \\
	&~\geq~ \Pr( \rho(t) = t-1, \rho(t+1) = t \mid \mathcal \rho(1), \ldots, \rho(t-1)) \cdot \Pr(\xi_t \geq-\tau) \nonumber \\
	&~\geq~ \frac{1-o(1)}{8} ~.
\end{align}
Conditioning on the event $E_{t-1}^c \cap\{ \rho(t) = t-1, \rho(t+1) = t, \xi_t \geq -\tau\}$, we note that $E_t$ is independent of $\mathcal F_{t-1}$.
Combined with \eqref{eq-lemma-6-1} and \eqref{eq-lemma-6-2}, it follows that
\begin{align*}
	\Pr(E_t \mid \mathcal F_{t-1}) 
	&~\geq~ \frac{1-o(1)}{8} \cdot \Pr(E_{t} \mid E_{t-1}^c, \rho(t) = t-1, \rho(t+1) = t, \xi_t \geq -\tau) \\
	&~\geq~ \frac{1-o(1)}{8} \cdot \Pr(|\xi_t| \leq \tau, \xi_{t+1} < -\tau \mid \xi_t \geq -\tau) \\
	&~=~ \Omega(\tau/\sigma) ~,
\end{align*}
where the last inequality follows from the assumption that $\tau<\sigma/\log T$.
\end{proof}

\begin{lemma}\label{lma:switchingPenalty} 
For all $t > 1$ it holds that
$\E[R_t - \tR_t \mid X_t \neq X_{t-1}, E_t] ~\geq~ \eta/3 - \tau$.
\end{lemma}

\begin{proof}
We rewrite $\E[R_t - \tR_t \mid X_t \neq X_{t-1}, E_t]$ as
\begin{align*}
	& \Pr(\Lambda_t = X_t \mid X_t \neq X_{t-1}, E_t)~ \E[R_t - \tR_t \mid X_t \neq X_{t-1}, E_t, \Lambda_t = X_t] \\
	& +~\Pr(\Lambda_t \neq X_t \mid X_t \neq X_{t-1}, E_t)~ \E[R_t - \tR_t \mid X_t \neq X_{t-1}, E_t, \Lambda_t \neq X_t] ~~.
\end{align*}
First consider the case where $\Lambda_t = X_t$, namely, the
orientation of the spikes coincides with the direction of the player's
switch. In this case,
$$
\E[R_t - \tR_t \mid X_t \neq X_{t-1}, E_t, \Lambda_t = X_t] ~\geq~ \eta - \tau ~~.
$$
If $\Lambda_t = X_t$ then the orientation of the spikes does not coincide with the switch direction and 
$$
\E[R_t - \tR_t \mid  X_t \neq X_{t-1}, E_t, \Lambda_t \neq X_t] ~\geq~ - \tau ~~.
$$
\lemref{lma:probRatio} below implies that 
$$
\Pr(\Lambda_t = X_t \mid X_t \neq X_{t-1}, E_t) ~\geq~ \frac{1}{3} ~~,
$$
which concludes the proof.
\end{proof}

\begin{lemma}\label{lma:probRatio}
Suppose that $\eta \leq \sigma/\log T$. For a sufficiently large $T$ it holds that 
$$
\frac{\Pr(\Lambda_t = X_t \mid X_t \neq X_{t-1}, E_t)}{\Pr(\Lambda_t \neq X_t \mid X_t \neq X_{t-1}, E_t)} ~\ge~ \half ~~.
$$
\end{lemma}

\begin{proof}
The ratio on the left can be rewritten, using Bayes' rule, as
$$
\frac{\Pr( X_t \neq X_{t-1} \mid \Lambda_t = X_t,   E_t)}{\Pr( X_t \neq X_{t-1} \mid \Lambda_t \ne X_t,   E_t)}   ~~.
$$
To see this is at least $\half$, condition on the history until time $t-2$ and note that by our earlier modification of the algorithm,
we may assume that $|L_{t-1}(X_{t-1})- L_{\rho(t-1)}(X_{\rho(t-1)}) | \le 4\sigma\sqrt{\log{T}}$. We let $f_1(x)$ be the conditional density of $L_{t-1}(X_{t-1})- L_{\rho(t-1)}(X_{\rho(t-1)})$ given $\{X_t =\Lambda_t\} \cap E_t$, and let $f_2(x)$ be the conditional density of $L_{t-1}(X_{t-1})- L_{\rho(t-1)}(X_{\rho(t-1)})$ given $\{\Lambda_t \neq X_t\} \cap E_t$. Therefore, we see that $f_1$ is the density function for $\sigma Z + \eta$, and $f_2$ is the density function for $\sigma Z$ where $Z \sim N(0, 1)$. Thus, we have
\begin{equation} \label{eq-two-scenario}
	\min_{|x| \leq 4 \sigma \sqrt{\log T}} \frac{f_1(x)}{f_2(x)} 
	~=~ \min_{|x| \leq 4 \sigma \sqrt{\log T}} \frac{\mathrm{e}^{-(x-\eta)^2/2\sigma^2}}{\mathrm{e}^{-x^2/2\sigma^2}} 
	~=~ 1-o(1) ~,
\end{equation}
where the last inequality follows from the assumption that $\eta \leq \sigma/\log T$. Now consider  two scenarios of the game where the observations are identical up to time $t-2$, and then for the two scenarios we condition on events $\{\Lambda_t = X_t\} \cap E_t$ and $\{\Lambda_t \neq X_t\} \cap E_t$ respectively. Then by \eqref{eq-two-scenario} the observation at time $t-1$ is statistically close, and therefore the algorithm will make a decision for $X_t$ that is statistically close in these two scenarios. Formally, we get that
$$
	\frac{\Pr( X_t \neq X_{t-1} \mid \Lambda_t = X_t,   E_t)}{\Pr( X_t \neq X_{t-1} \mid \Lambda_t \ne X_t,   E_t)}   
	~\geq~ \min_{|x| \leq 4 \sigma \sqrt{\log T}} \frac{f_1(x)}{f_2(x)} 
	~=~ 1-o(1) ~,
$$
completing the proof of the lemma.
\end{proof}

\subsection{The Max Adversary} \label{sec:maxAdversary}

In the previous section, we proved that the minimax regret, with
bandit feedback, against the min adversary is $\t{\Omega}(T^{2/3})$. 
The same can be proved for the max adversary, using
an almost identical proof technique, namely, by using the max
combining function to simulate a switching cost. The construction of
the loss process $Z_{1:T}$ remains as defined above. The event $E_t$
changes, and requires $|W_{t-1}-W_t| \leq \tau$ and $W_{t+1} > W_t +
\eta$. The spikes also change: we set 
$$
S_{t-1}(\Lambda_t) = 1, \quad S_{t}(\Lambda_t) = 1, \quad S_{t-1}(1-\Lambda_t) = 0, \quad S_{t}(1-\Lambda_t) = 0~~.
$$
%The formal proof is omitted due to space constraints.
The formal proof is omitted.

\section{Linear Composite Functions are Easy} \label{sec:linear}

In this section, we consider composite functions that are linear in
the oblivious function $\ell_{t-m:t}$.  Namely, the adversary 
chooses a memory size $m \geq 1$ and defines
\begin{align} \label{eq:linear-func}
\forall\,t\quad f_t(x_{1:t}) ~=~ a_{m} \ell_{t-m}(x_{t-m}) + \cdots + a_{0} \ell_t(x_t) ~~,
\end{align}
where $a_{0},a_{1},\ldots,a_{m}$ are fixed, bounded, and known
coefficients, at least one of which is non-zero (otherwise the regret
is trivially zero). In order to ensure that $f_{t}(x_{1:t}) \in [0,1]$
for all $t$, we assume that $\sum_{i=0}^{m} a_{i} \le 1$. We can also
assume, without loss of generality, that in fact $\sum_{i=0}^{m} a_{i}
= 1$, since scaling all of the loss functions by a constant scales the
regret by the same constant. Recall that, for completeness, we assumed
that $\ell_{t} \equiv 0$ for $t \le 0$.

\begin{algorithm}
\caption{\sc Strategy for Linear Composite Functions} \label{alg:linear}
\begin{algorithmic}
\STATE set $d = \min\{ i \ge 0 \,:\, a_{i} \ne 0 \}$
\STATE initialize $d+1$ independent instances $\mathcal{A}_{0},\ldots,\mathcal{A}_{d}$ of \textsc{Exp3}.
\STATE initialize $z_{0} = z_{-1} = \ldots = z_{-m+1} = 0$
\FOR{$t = 1$ to $T$}
	\STATE set $j = t \mod (d+1)$
	\STATE draw $x_{t} \sim \mathcal{A}_{j}$
	\STATE play $x_{t}$ and observe feedback $f_{t}(x_{1:t})$
	\STATE set
	$
		z_{t} ~\gets~ 
			\frac{1}{a_{d}} \left( f_{t}(x_{1:t}) - \sum_{i=d+1}^{m} a_{i} z_{t-i} \right)
	$
	\STATE feed $\mathcal{A}_{j}$ with feedback $z_{t}$ (for action $x_{t}$)
\ENDFOR
\end{algorithmic}
\end{algorithm}

We show that an adversary that chooses a linear composite loss induces
an \emph{easy} bandit learning problem. More specifically, we present
a strategy, given in Algorithm~\ref{alg:linear}, that achieves
$\t{O}(\sqrt{T})$ regret against any loss function sequence of
this type. This strategy uses the \textsc{Exp3} algorithm
\citep{Auer:02} as a black box, and relies on the guarantee that
\textsc{Exp3} attains a regret of $O(\sqrt{Tk \log T})$ against
any oblivious loss sequence, with bandit feedback. 
\begin{theorem} \label{thm:linear}
For any sequence of loss functions $f_{1:T}$ of the form in \eqref{eq:linear-func}, the expected regret of Algorithm~\ref{alg:linear} satisfies
$R = O(\sqrt{m T k\log k})$.
\end{theorem}
\begin{proof}
First, observe that $z_{t} = \ell_{t-d}(x_{t-d})$ for all $t \in [T]$.
Indeed, for $t = 1$ this follows directly from the definition of $f_{1:m}$ (and from the fact that $z_{t} = 0$ for $t \le 0$), and for $t > 1$ an inductive argument shows that
\begin{align*}
	z_{t} 
	~=~ \frac{1}{a_{d}} \left( f_{t}(x_{1:t}) - \sum_{i=d+1}^{m} a_{i} z_{t-i} \right)
	~=~ \frac{1}{a_{d}} \left( f_{t}(x_{1:t}) - \sum_{i=d+1}^{m} a_{i} \ell_{t-i}(x_{t-i}) \right)
	~=~ \ell_{t-d}(x_{t-d}) ~.
\end{align*}
Hence, each algorithm $\mathcal{A}_{j}$ actually plays a standard bandit game with the subsampled sequence of oblivious loss functions $\ell_{j}, \ell_{j+(d+1)}, \ell_{j+2(d+1)}, \ldots$ .
Consequently, for each $j = 0,1,\dots,d$ we have
\begin{align} \label{eq:exp3j}
	\forall ~ x \in [k] ~, \quad
	\E \left[ \sum_{t \in S_{j}} \ell_{t}(x_{t}) \right] - \sum_{t \in S_{j}} \ell_{t}(x)
	~=~ O\left( \sqrt{\frac{Tk\log k}{d}}\right) ~~,
\end{align}
where $S_{j} = \{ t \in [T] ~:~ t = j \mod (d+1) \}$, and we used the fact that $|S_{j}| = \Theta(T/d)$. 
Since the sets $S_{0},\ldots,S_{d}$ are disjoint and their union equals $[T]$, by summing \eqref{eq:exp3j} over $j=0,1,\ldots,d$ we obtain
\begin{align} \label{eq:my-exp3}
	\forall ~ x \in [k] ~, \quad
	\E \left[ \sum_{t=1}^{T} \ell_{t}(x_{t}) \right] - \sum_{t=1}^{T} \ell_{t}(x)
	~=~ O( \sqrt{d T k\log k} ) 
	~=~ O( \sqrt{m T k\log k} )~~.
\end{align}
However, notice that the loss of the player satisfies
\begin{align*}
	\sum_{t=1}^{T} f_{t}(x_{1:t}) 
	&~=~ \sum_{t=1}^{T} \big( a_{m} \ell_{t-m}(x_{t-m}) + \cdots + a_{0} \ell_t(x_t) \big) \\
	&~\le~ a_{m} \sum_{t=1}^{T}  \ell_{t}(x_{t}) + \cdots
		+ a_{0} \sum_{t=1}^{T}  \ell_{t}(x_{t}) \\
	&~=~ \sum_{t=1}^{T} \ell_{t}(x_{t}) ~,
\end{align*}
where the last equality uses the assumption that $\sum_{i=0}^{m} a_{i} = 1$.
A similar calculation shows that for any fixed $x \in [k]$,
\begin{align*}
	\sum_{t=1}^{T} f_{t}(x,\ldots,x) 	
	&~=~ \sum_{t=1}^{T} \big( a_{m} \ell_{t-m}(x) + \cdots + a_{0} \ell_t(x) \big) \\
	&~\ge~ a_{m} \left( \sum_{t=1}^{T} \ell_{t}(x) - m \right)+ \cdots
		+ a_{0} \left( \sum_{t=1}^{T} \ell_{t}(x) - m \right) \\
%	&~\ge~ \sum_{t=1}^{m} \ell_{t}(x) 
%		+ a \cdot \left( \sum_{t=1}^{T} \ell_{t}(x) - m \right) \\
	&~=~ \sum_{t=1}^{T} \ell_{t}(x) - m ~.
\end{align*}
Putting things together, we obtain that for all $x \in [k]$,
\begin{align*}
\sum_{t=1}^{T} f_{t}(x_{1:t}) - \sum_{t=1}^{T} f_{t}(x,\ldots,x)
	&~\le~ \sum_{t=1}^{T} \ell_{t}(x_{t}) - \sum_{t=1}^{T} \ell_{t}(x) + m ~.
\end{align*}
Finally, taking the expectation of this inequality and combining with \eqref{eq:my-exp3} completes the proof.
\end{proof}

\section{Conclusion}  \label{sec:conc}

\citet{CesaBianchiDeSh13,DekelDiKoPe14} were the first to show that a
finite-horizon online bandit problem with a finite set of actions can
be \emph{hard}. They achieved this by proving that the minimax regret
of the multi-armed bandit with switching costs has a rate of
$\t{\Theta}(T^{2/3})$. In this paper, we defined the
class of online learning problems that define their loss values using
a composite loss function, and proved that two non-linear instances
of this problem are also hard. Although we reused some technical
components from the analysis in \citet{DekelDiKoPe14}, the composite
loss function setting is quite distinct from the multi-armed bandit
with switching costs, as it does not explicitly penalize
switching. Our result reinforces the idea that the class of hard
online learning problems may be a rich class, which contains many
different natural settings. To confirm this, we must discover
additional online learning settings that are provably hard.

We also proved that linear composite functions induce easy bandit
learning problems. Characterizing the set of combining functions that
induce hard problems remains an open problem.

\bibliographystyle{abbrvnat}
\bibliography{bib}

%%\appendix
%%
%%\section{Technical Proofs}

\end{document}

%% file: macro.tex
\renewcommand{\Pr}{\mathbb{P}}

\newcommand{\tR}{\tilde R}

\newcommand{\E}{\mathbb{E}}

\newcommand{\half}{\frac{1}{2}}

\newcommand{\Fcal}{\mathcal{F}}

\newcommand{\Xcal}{\mathcal{X}}

\newcommand{\secref}[1]{Sec.~\ref{#1}}

\renewcommand{\eqref}[1]{Eq.~(\ref{#1})}
\newcommand{\lemref}[1]{Lemma~\ref{#1}}
\newcommand{\thmref}[1]{Theorem~\ref{#1}}

\newcommand{\ignore}[1]{}

\newcommand{\abs}[1]{\left|#1\right|}
\newcommand{\lr}[1]{\left(#1\right)}
\newcommand{\set}[1]{\left\{#1\right\}}
\renewcommand{\t}[1]{\smash{\widetilde{#1}}} % made wide YP

\newcommand{\ind}[1]{1\!\!1_{#1}}

\newcommand{\eps}{\epsilon}

\newcommand{\Q}{\mathcal{Q}}

\newcommand{\sig}{\sigma}

\newcommand{\cut}{\mathrm{cut}}
\newcommand{\anc}{\rho^*}
\newcommand{\depth}{d}
\newcommand{\width}{w}
\newcommand{\clip}{\mathrm{clip}}